\documentclass{article}

\usepackage[nonatbib,preprint]{neurips_2019}

\usepackage[utf8]{inputenc} 
\usepackage[T1]{fontenc}    
\usepackage{hyperref}       
\usepackage{url}            
\usepackage{booktabs}       
\usepackage{amsfonts}       
\usepackage{nicefrac}       
\usepackage{microtype}      

\usepackage{amsmath}
\usepackage{amsthm}
\usepackage{todonotes}
\usepackage{amsfonts} 
\usepackage{algorithm}
\usepackage{algorithmic}
\usepackage{adjustbox}
\usepackage{multirow}
\usepackage{xspace}
\usepackage{subcaption} 
\usepackage{enumitem}

\newtheorem{theorem}{Theorem}
\newtheorem{lemma}{Lemma}
\newtheorem{corollary}{Corollary}
\newtheorem{definition}{Definition}
\newtheorem*{remark}{Remark}
\newtheorem*{goal}{Goal}
\newcommand{\tf}[1]{\mathbf{#1}}
\newcommand{\ttf}[1]{\boldsymbol{#1}}
\newcommand{\SigSp}[1]{\mathcal{#1}}

\title{Verification of Neural Network Control Policy Under Persistent Adversarial Perturbation}

\author{%
  Yuh-Shyang Wang \\
  GE Research\\
  \texttt{yuh-shyang.wang@ge.com} \\
  \And
  Tsui-Wei Weng \\
  MIT \\ 
  \texttt{twweng@mit.edu} \\
  \And
  Luca Daniel \\
  MIT \\
  \texttt{luca@mit.edu} \\  
}

\begin{document}

\maketitle

\begin{abstract}
Deep neural networks are known to be fragile to small adversarial perturbations. This issue becomes more critical when a neural network is interconnected with a physical system in a closed loop. In this paper, we show how to combine recent works on neural network certification tools (which are mainly used in static settings such as image classification) with robust control theory to certify a neural network policy in a control loop. Specifically, we give a sufficient condition and an algorithm to ensure that the closed loop state and control constraints are satisfied when the persistent adversarial perturbation is $\ell_\infty$ norm bounded. Our method is based on finding a positively invariant set of the closed loop dynamical system, and thus we do not require the differentiability or the continuity of the neural network policy. Along with the verification result, we also develop an effective attack strategy for neural network control systems that outperforms exhaustive Monte-Carlo search significantly. We show that our certification algorithm works well on learned models and achieves $5$ times better result than the traditional Lipschitz-based method to certify the robustness of a neural network policy on a cart pole control problem.
\end{abstract}

\section{Introduction} \label{submission}

Deep neural networks (DNN) have become state-of-the-arts in a variety of machine learning tasks, including control of a physical system in the reinforcement learning setting. For example, in guided policy search~\cite{zhang2016learning}, a neural network policy is used to replace the online model predictive control policy to directly control a physical system in a feedback loop. Yet, recent studies demonstrate that neural networks are surprisingly fragile - the neural network policies~\cite{huang2017adversarial}, neural network classifiers~\cite{szegedy2013intriguing} and many other tasks~\cite{xie2017adversarial, jia2017adversarial, cisse2017houdini} are all vulnerable to adversarial examples and attacks. For applications that are safety-critical, such as self-driving cars, the existence of adversarial examples in neural networks has raised severe and unprecedented concerns due to recent popularity on deploying neural network based policies in various physical systems. A small perturbation on the input of the neural network may cause significant change in control actions, which could possibly destabilize the system or drive the system state to an unsafe region. Therefore, it is crucial to develop a verification tool that can provide useful certificate (such as safety or robustness guarantees) for a dynamical system with a neural network policy in the loop.

Recently, much research effort has been devoted to developing verification methods to quantify the robustness of neural networks against adversarial input perturbations. The key idea is to characterize the set of neural network output when the adversarial perturbations are constrained in a given set, usually a weighted $\ell_p$ ball. Current mainstream verification methods can be categorized to exact verifiers~\cite{katz2017reluplex} and inexact verifiers~\cite{kolter2017provable,weng2018towards,Gehr2018AI2,dvijotham2018dual,Boopathy2019cnncert, royo2019fast} based on the characterization of output reachable set being \textit{exact}  or \textit{over-approximated}. Exact verifiers are hard to scale to large and deep neural networks due to its intrinsic NP-completeness~\cite{katz2017reluplex}; hence often inexact verifiers are more favorable in real neural network applications due to their scalability and capability to deliver non-trivial robustness certificate efficiently. Nevertheless, all the above verifiers are developed in a \textit{static} setting, specifically for neural network classifiers, where the neural networks do not interact with a dynamical system. 
As neural network policies are deployed in many real-world systems~\cite{zhang2016learning, kiumarsi2017optimal}, it is necessary to extend the \emph{static} neural network certification tool to a \emph{dynamic} setting, so one can certify a neural network policy in a feedback control loop.



In this paper, we propose a novel framework to verify neural network policies in a closed loop system by combining the aforementioned static neural network certification tools with robust control theory. We consider a realistic and strong threat model where the adversaries can manipulate observations and states at \textit{every time step over an infinite horizon} (hence the so-called \textit{persistent} perturbations). Our key idea is to leverage the static neural network certification tools to give a tight input-output characterization of the neural network policy, that when combine with robust control theory, allow us to find a positively invariant set of the closed loop dynamical system. Our contributions are summarized as below.
\paragraph{Contributions.}
\begin{itemize}[leftmargin=*]
\item To our best knowledge, our work is the first one to extend the neural network certification tools~\cite{katz2017reluplex, kolter2017provable,weng2018towards,Gehr2018AI2,dvijotham2018dual,Boopathy2019cnncert, royo2019fast} to a dynamic setting, in which we certify a neural network policy in a feedback control loop under persistent adversarial perturbation.
\item Our framework can handle non-Lipschitz and discontinuous neural network policies, while the traditional Lipschitz-based robust control approach cannot. Even when the neural network policy is Lipschitz continuous, we prove theoretically and validate experimentally that our proposed framework is \emph{always} better than the traditional Lipschitz-based robust control approach in terms of tighter certification bound.
\item We demonstrate that our method works well on situations where the system dynamic is unknown, unstable, and nonlinear. As long as one can learn the model and \emph{over-approximate} the modeling error using a data driven approach, our certification algorithm can be applied.
\item Along with the certification algorithm, we develop an effective persistent $\ell_\infty$ attack algorithm that can successfully discover the vulnerability of a neural network control system and show that such vulnerability cannot be found by exhaustive Monte-Carlo simulation. This observation indicates that the robustness of a neural network control system cannot be certified via exhaustive simulation, and the mathematical-based certification framework developed in this paper is necessary.
\end{itemize}






\paragraph{Outline.} We organize the paper as follows. We summarize related works and notations by the end of this section, and then describe problem setting and formulations in Section \ref{sec:formulation}. We present our main theorems and algorithms in Sections \ref{sec:boundedness} and \ref{sec:find_invariant}, where we propose an algorithm to find a certificate to ensure that the closed loop state and control constraints are satisfied when the persistent adversarial attack signal is $\ell_\infty$ norm bounded. In Section \ref{sec:experiment}, we conduct a comprehensive case study and show that our proposed algorithms not only deliver much stronger certificates than traditional robust control approaches but also work well on both learned models and non-Lipschitz neural network policies.


\paragraph{Related works.} Our work is closely related to the neural network certification tools mentioned above. Specifically, given a neural network policy $u = \pi(y)$, these tools can certify that $u = \pi(y) \in \SigSp{U}$ for all $y \in \SigSp{Y}$ for some sets $\SigSp{U}$ and $\SigSp{Y}$. We extend these tools to a dynamic setting where the neural network output $u$ can affect the neural network input $y$ in the future via a feedback loop. Another closely related field is safe reinforcement learning~\cite{garcia2015comprehensive, chow2018lyapunov}, where the safety during policy exploration and deployment are investigated. Recent works in~\cite{berkenkamp2017safe, richards2018lyapunov, jin2018stability, jin2018control, ivanov2019verisig} can further give a certificate of stability or safety properties of a neural network control system.
Our work differs from the above works in two aspects. First, we explicitly incorporate the adversarial perturbation in our formulation, especially the \emph{persistent perturbation sequence} with bounded $\ell_\infty$ norm (the $\ell_2$ norm and the energy of the signal are unbounded). The Lyapunov based method~\cite{richards2018lyapunov} and the hybrid system method~\cite{ivanov2019verisig} do not deal with perturbation and the integral quadratic constraint based method~\cite{jin2018stability, jin2018control} is developed with bounded $\ell_2$ norm on the perturbation. The second difference is that our method does not require Lipschitz continuity on the neural network policy, a common assumption in most of the literature. Therefore, our method can be applied on neural networks that are discontinuous or non-differentiable in nature -- due to quantization, digitalization, switching logic, obfuscated gradient~\cite{athalye2018obfuscated}, or other defense strategies. Even when the neural network policy is Lipschitz continuous, we prove theoretically and validate experimentally that our method can give a tighter bound than the Lipschitz-based robust control method. Another relevant work is the reachability analysis of a neural network control system over a finite time horizon~\cite{xiang2018reachability}. In this paper, we derive safety guarantee over an infinite time horizon with tighter bounds.




\paragraph{Notations.} We use lower case letters such as $x$ to denote vectors and upper case letters such as $A$ to denote matrices. We use $x \preceq y$ to denote that $x$ is element-wisely less than or equal to $y$. For a square matrix $A$, we use $\rho(A)$ to denote its spectral radius, which is the largest absolute value of the eigenvalues of $A$. We use boldface letters such as $\tf x$ and $\tf{\Phi}$ to denote signals and transfer matrices in the frequency domain, respectively. Consider a time series signal $\{x[t]\}_{t=0}^\infty$, the unilateral $z$-transform of the time series $\{x[t]\}_{t=0}^\infty$ is given by $\tf x (z) = \sum_{t=0}^\infty z^{-t} x[t]$. We use $\Phi[t]$ to denote the $t$-th spectral component of a transfer matrix $\tf{\Phi}$, i.e., $\tf{\Phi}(z) = \sum_{t=0}^\infty z^{-t} \Phi[t]$. We use $\tf x$ and $\tf{\Phi}$ as shorthand of $\tf x(z)$ and $\tf{\Phi}(z)$ when the context makes it clear. 
For a frequency domain equation $\tf y = \tf{\Phi} \tf x$, the corresponding time domain equation is given by the convolution formula $y[t] = \sum_{\tau=0}^\infty \Phi[\tau] x[t-\tau]$. We use $\| \tf x \|_{\ell_p}$ to denote the standard $\ell_p$ norm of the signal $\tf x$ and use $\| \tf{\Phi} \|_{\mathcal{L}_1}$ and $\| \tf{\Phi} \|_{\mathcal{H}_\infty}$ to denote the $\mathcal{L}_1$ norm~\cite{dahleh19871} and the $\mathcal{H}_\infty$ norm~\cite{zhou1996robust} of the system $\tf{\Phi}$. We define the absolute operator $\text{abs}(\cdot)$ of a stable transfer matrix $\tf{\Phi}$ by $\text{abs}(\tf{\Phi}) = \sum_{t=0}^\infty | \Phi[t] |$. Note that each element of the transfer matrix $\tf{\Phi}$ is absolutely summable if and only if the transfer matrix $\tf{\Phi}$ is real rational and stable (see Page 113 - 114 of \cite{oppenheim1996signals}). 

\section{Problem Formulation} \label{sec:formulation}
\begin{figure}[h!]
      \centering
      \includegraphics[width=0.7\textwidth]{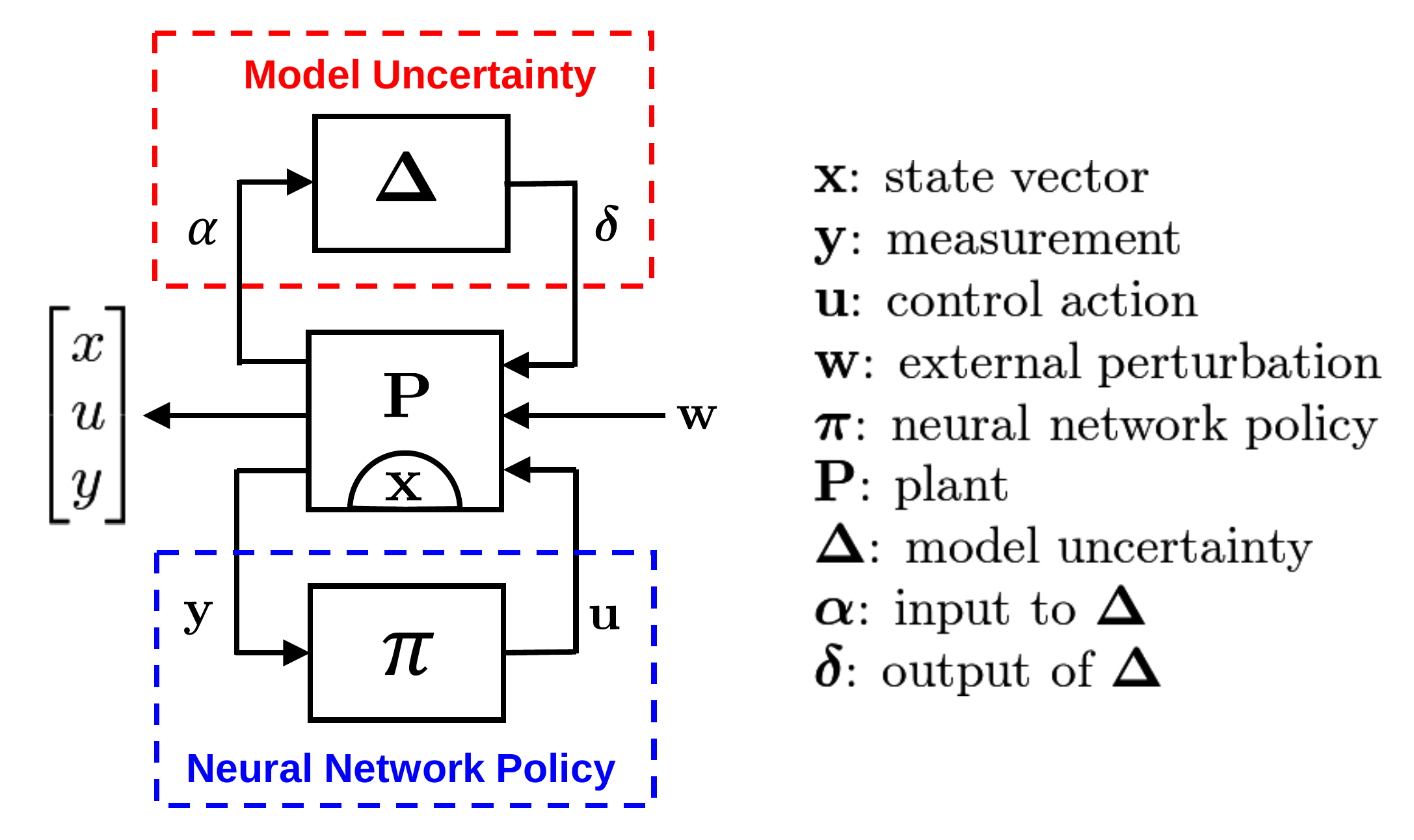}
      \caption{A neural network policy interconnected with an uncertain dynamical system.}
      \label{fig:model}
\end{figure}
We consider a neural network policy interconnected with a dynamical system. 
The model architecture is shown in Figure \ref{fig:model}, where $\pi(\cdot)$ is the neural network control policy, $\tf{P}$ is the plant to be controlled, and $\ttf{\Delta}$ characterizes the uncertainty of the model $\tf{P}$. When the model dynamics is unknown in the first place, one can use data driven approaches such as~\cite{dean2017sample, recht2019tour} to learn the nominal model $\tf{P}$ and over-approximate the modeling error using $\ttf{\Delta}$. We assume that $\tf{P}$ is a discrete time linear time invariant (LTI) system with dynamics given by
\begin{eqnarray}
x[t+1] &=& A x[t] + B u[t] + B_w w[t] + B_{\delta} \delta[t] \label{eq:NLTI1}\\
y[t] &=& C x[t] + D_w w[t] \label{eq:NLTI2} \\
\alpha[t] &=& C_{\alpha} x[t] + D_{\alpha u} u[t] + D_{\alpha w} w[t]. \label{eq:NLTI3}
\end{eqnarray}
where $t$ denotes the time index, $x$ the state vector, $u$ the control action, $y$ the measurement, $w$ the external perturbation, and $\alpha$ and $\delta$ the input and output of the uncertainty block $\ttf{\Delta}$. In particular, $w$ is a persistent perturbation over an infinite horizon $t \geq 0$. In the rest of the paper, we assume zero initial condition unless stated otherwise:
\begin{equation}
    x[0] = 0. \label{eq:zero}
\end{equation}
This is to simplify the presentation -- the method proposed in this paper can be readily extended to handle persistent perturbation with nonzero initial condition as well. We assume the pair $(A,B)$ is stabilizable and $(A,C)$ is detectable. 
The neural network policy is described by
\begin{equation}
    u[t] = \pi(y[t]), \label{eq:NLTI4}
\end{equation}
which is assumed to be a static policy in this paper. The uncertainty block can be dynamic, with 
\begin{equation}
    \ttf{\delta} = \ttf{\Delta}(\ttf{\alpha}) \label{eq:NLTI5}
\end{equation}
in the frequency domain. We assume that $\ttf{\Delta}$ is stable and norm bounded. Equations \eqref{eq:NLTI1} - \eqref{eq:NLTI5} give a complete description of the model architecture shown in Figure \ref{fig:model}. 

\begin{remark}
If the input to the neural network policy is a finite horizon of historical measurement, i.e., $u[t] = \pi(y[t], y[t-1],\dots,y[t-T])$ for a finite length $T$, then we can augment the state $x$ and measurement $y$ for $T$ steps and convert $\pi(\cdot)$ to a static policy.
\end{remark}
\begin{remark}
Note that $y$ in \eqref{eq:NLTI2} is not a direct function of $u$ and $\delta$, and $\alpha$ in \eqref{eq:NLTI3} is not a direct function of $\delta$. This assumption ensures that the feedback structure in Figure \ref{fig:model} is well-posed because there is no algebraic loop in the equation. 
\end{remark}

In this paper, we use $\tf w$ to model the persistent adversarial perturbation.
We assume that $\tf w$ lies in the set $\|\tf w\|_{\ell_\infty} \leq w_\infty$, or equivalently,
\begin{equation}
    \tf w \in \{ \tf w \,\, | \,\, |w[t]| \preceq w_\infty \mathbf{1} = \bar{w}, \,\, \forall t \geq 0 \}. \label{eq:attack}
\end{equation}
The matrices $B_w$ and $D_w$ in \eqref{eq:NLTI1} - \eqref{eq:NLTI2} control the impact of $\tf w$ on the state $\tf x$ and the neural network input $\tf y$.
We note that the block $\ttf{\Delta}$ can be used to model another type of adversarial perturbation, such as changing the physical parameters of the model. 
The goal of this paper is the following: 
\begin{goal}
Given the model equations \eqref{eq:NLTI1} - \eqref{eq:NLTI5}, design an algorithm to certify that the following requirement is satisfied for any adversarial sequence $\{w[t]\}_{t=0}^\infty$ satisfying \eqref{eq:attack}: 
\begin{equation}
    |x[t]| \preceq x_{lim}, \quad |y[t]| \preceq y_{lim}, \quad |u[t]| \preceq u_{lim}, \quad \forall t \geq 0. \label{eq:goal}
\end{equation}
\end{goal}

\section{Closed Loop Boundedness} \label{sec:boundedness}

In this section, we derive a sufficient condition to ensure that the state $\tf x$, measurement $\tf y$, and control $\tf u$ in Figure \ref{fig:model} are bounded for any adversarial attack $\tf w$ lies within the $\ell_\infty$ ball \eqref{eq:attack}. 
In Section \ref{sec:open_stable}, we assume that the plant model \eqref{eq:NLTI1} - \eqref{eq:NLTI3} is open loop stable, i.e., $\rho(A) < 1$, and directly apply traditional robust control theory to obtain Lemma \ref{lemma:L1} for closed loop stability. We then explain the restrictions of Lemma \ref{lemma:L1} and improve the results by combining robust control theory with neural network certification tools to obtain Theorem \ref{thm:LTI} in Section \ref{sec:main_thm}. Finally, we extend our results to unstable plants in Section \ref{sec:open_unstable}.


\subsection{Stable Plant: robust control baseline} \label{sec:open_stable}
With the zero initial condition assumption, we can use the frequency domain notation introduced before to rewrite \eqref{eq:NLTI1} - \eqref{eq:NLTI3} as
\begin{subequations}
\begin{align}
\tf{x} &= \tf{\Phi_{xu}} \tf u + \tf{\Phi_{xw}} \tf w + \tf{\Phi_{x \ttf{\delta}}} \ttf{\delta} \label{eq:FNLTI1}\\
\tf{y} &= \tf{\Phi_{yu}} \tf u + \tf{\Phi_{yw}} \tf w + \tf{\Phi_{y \ttf{\delta}}} \ttf{\delta} \label{eq:FNLTI2}\\
\ttf{\alpha} &= \tf{\Phi_{\ttf{\alpha} u}} \tf u + \tf{\Phi_{\ttf{\alpha} w}} \tf w + \tf{\Phi_{\ttf{\alpha} \ttf{\delta}}} \ttf{\delta}. \label{eq:FNLTI3}
\end{align}
\end{subequations}
with $\tf{\Phi_{xu}} = (zI-A)^{-1} B$, $\tf{\Phi_{xw}} = (zI-A)^{-1} B_w$, $\tf{\Phi_{x \ttf{\delta}}} = (zI-A)^{-1} B_\delta$, $\tf{\Phi_{yu}} = C (zI-A)^{-1} B$, $\tf{\Phi_{yw}} = C (zI-A)^{-1} B_w + D_w$,  $\tf{\Phi_{y \ttf{\delta}}} = C(zI-A)^{-1} B_\delta$, $\tf{\Phi_{\alpha u}} = C_{\alpha}(zI-A)^{-1} B + D_{\ttf{\alpha} u}$, $\tf{\Phi_{\ttf{\alpha} w}} = C_{\alpha}(zI-A)^{-1} B_w + D_{\alpha w}$, and $\tf{\Phi_{\ttf{\alpha} \ttf{\delta}}} = C_{\alpha}(zI-A)^{-1} B_\delta$. 

To show the input-output stability\footnote{The closed loop system is said to be finite gain input-output stable if the gain from perturbation $\tf w$ to $(\tf x, \tf u, \tf y)$ is finite.} of the structure in Figure \ref{fig:model}, we first assume that the neural network policy is locally Lipschitz continuous with a finite $\ell_\infty$ to $\ell_\infty$ gain $\gamma_\pi$, i.e., $\|u \|_{\ell_\infty} \leq \gamma_\pi \|y \|_{\ell_\infty}$ over some range $\|y\|_{\ell_\infty} \leq y_\infty$. For a stable transfer matrix, the $\ell_\infty$ to $\ell_\infty$ induced norm is known as the $\mathcal{L}_1$ system norm~\cite{dahleh19871}, which is defined by
\begin{equation}
    \|\tf{G}\|_{\mathcal{L}_1} = \underset{i}{\text{max}} \sum_{j=1}^n \|\tf{g_{ij}}\|_1 = \underset{i}{\text{max}} \sum_{j=1}^n \sum_{t=0}^\infty |g_{ij}[t]| \nonumber
\end{equation}
where $\tf{g_{ij}}$ is the $(i,j)$-th entry of the transfer matrix $\tf{G}$ and $n$ is the number of column of $\tf{G}$. 
The following Lemma gives the local input-output stability of Figure \ref{fig:model}, which directly comes from robust control theory:
\begin{lemma} \label{lemma:L1}
Consider a stable LTI plant \eqref{eq:NLTI1} - \eqref{eq:zero} interconnected with a neural network policy \eqref{eq:NLTI4} and a dynamic uncertainty block \eqref{eq:NLTI5} as shown in Figure \ref{fig:model}. Assume that the persistent perturbation $\tf w$ lies in the set given by \eqref{eq:attack}.
Suppose that the neural network policy $u = \pi(y)$ has a finite $\ell_\infty$ to $\ell_\infty$ gain $\gamma_\pi$ for all $\|y\|_{\ell_\infty} \leq y_\infty$, and the uncertainty block $\ttf{\Delta}$ has the property $\| \ttf{\delta} \|_{\ell_\infty} \leq \gamma_\Delta \|\ttf{\alpha} \|_{\ell_\infty}$. If the following three conditions hold:
\begin{subequations}
\begin{align}
& \beta_1 = \gamma_{\Delta} \| \tf{\Phi_{\ttf{\alpha} \ttf{\delta}}} \|_{\mathcal{L}_1} < 1 \label{eq:L1-1}\\
& \beta_2 = \gamma_\pi \Big[\|\tf{\Phi_{yu}}\|_{\mathcal{L}_1} + \frac{\gamma_{\Delta}}{1 - \beta_1} \| \tf{\Phi_{y \ttf{\delta}}}\|_{\mathcal{L}_1} \|\tf{\Phi_{\ttf{\alpha} u}}\|_{\mathcal{L}_1} \Big] < 1 \label{eq:L1-2} \\
& \frac{1}{1-\beta_2} \Big[ \| \tf{\Phi_{yw}} \|_{\mathcal{L}_1} + \frac{\gamma_{\Delta}}{1 - \beta_1} \| \tf{\Phi_{y \ttf{\delta}}}\|_{\mathcal{L}_1} \| \tf{\Phi_{\ttf{\alpha} w}} \|_{\mathcal{L}_1} \Big] w_\infty \leq y_\infty \label{eq:L1-3}
\end{align}
\end{subequations}
then the closed loop system in Figure \ref{fig:model} is input-output stable over the region $\|y\|_{\ell_\infty} \leq y_\infty$ for all adversarial attack $\tf{w}$ satisfying \eqref{eq:attack}.
\end{lemma}
The proof of Lemma \ref{lemma:L1} can be found in Appendix A. 
To certify the local input-output stabilit of the closed loop system using Lemma \ref{lemma:L1}, we can iteratively search for $y_\infty$ until the local Lipschitz continuous assumption and \eqref{eq:L1-1} - \eqref{eq:L1-3} are both satisfied, or the state constraint is violated. 
Note that when $\gamma_\pi$ is the \emph{global} $\ell_\infty$ to $\ell_\infty$ bound of the neural network policy, we can drop the condition \eqref{eq:L1-3} because $y_\infty$ can be arbitrarily large. The global version of Lemma \ref{lemma:L1}, i.e., with $y_\infty$ being arbitrarily large, can be found in the robust control literature in~\cite{khammash1991performance}.
When the magnitude of the signals $\tf w$ and $\tf y$ is characterized by the $\ell_2$ norm, we can replace the $\mathcal{L}_1$ norm in Lemma \ref{lemma:L1} by the $\mathcal{H}_\infty$ norm, as the $\mathcal{H}_\infty$ norm is equivalent to the $\ell_2$ to $\ell_2$ induced norm. We include the global and $\mathcal{H}_\infty$ version of Lemma \ref{lemma:L1} as a corollary in Appendix A, which is an unstructured version of the main loop theorem\footnote{This can be considered as an extension of the well-known small gain theorem.}~\cite{packard1993complex, zhou1996robust} in the robust control literature ($\mathcal{H}_\infty$ control and structured singular value). 

\subsection{Stable Plant: improvement} \label{sec:main_thm}
Lemma \ref{lemma:L1} has several restrictions.
First, the neural network policy needs to be Lipschitz continuous, and thus Lemma \ref{lemma:L1} cannot be applied on non-differentiable or discontinuous policies due to quantization or other issues. Second, even when the given neural network policy is Lipschitz continuous, the bound of the local Lipschitz constant $\gamma_\pi$ for a deep neural network policy is usually very loose. As a consequence, Lemma \ref{lemma:L1} can only be applied to certify the robustness of Figure \ref{fig:model} over a small region near the stable equilibrium. Finally and most importantly, even if the given policy is Lipschitz continuous and the local Lipschitz constant $\gamma_\pi$ is \emph{exact}, using $\| \tf u \|_{\ell_\infty} \leq \gamma_\pi \|\tf y \|_{\ell_\infty}$ to characterize the input-output relation of a given neural network is usually very loose. In this subsection, we improve the results of Lemma \ref{lemma:L1} and propose a more useful Theorem to certify the boundedness of the structure in Figure \ref{fig:model}.

Our strategy is to use the \emph{static} neural network certification tool to give a tighter characterization of the input-output relation of the given neural network policy. 
The following Theorem gives a sufficient condition to ensure closed loop boundedness under any adversarial attack satisfying \eqref{eq:attack}.


\begin{theorem} \label{thm:LTI}
Consider a stable LTI plant \eqref{eq:NLTI1} - \eqref{eq:zero} interconnected with a neural network policy \eqref{eq:NLTI4} and a dynamic uncertainty block \eqref{eq:NLTI5} as shown in Figure \ref{fig:model}. Assume that the persistent perturbation $\tf w$ lies in the set given by \eqref{eq:attack}. If we can find a quadruplet $(\bar{y}, \bar{u}, \bar{\alpha}, \bar{\delta})$ satisfying the following conditions:
\begin{enumerate}
    \item Neural network policy robustness certificate: The static policy $u = \pi(y)$ has the property $|u| \preceq \bar{u}$ for all $|y| \preceq \bar{y}$
    \item Input-output relation of the uncertainty block: $\ttf{\delta} = \Delta(\ttf{\alpha})$ has the property $| \alpha[k] | \preceq \bar{\alpha}$, $k = 0, \cdots, t$ $\implies$ $| \delta[t] | \preceq \bar{\delta}$ for all $t \geq 0$.
    \item Feedback condition: \text{abs}($\tf{\Phi_{yw}}$) $\bar{w}$ + \text{abs}($\tf{\Phi_{yu}}$) $\bar{u}$ + \text{abs}($\tf{\Phi_{y\ttf{\delta}}}$) $\bar{\delta}$ $\preceq \bar{y}$ and \text{abs}($\tf{\Phi_{\ttf{\alpha}w}}$) $\bar{w}$ + \text{abs}($\tf{\Phi_{\ttf{\alpha}u}}$) $\bar{u}$ + \text{abs}($\tf{\Phi_{\ttf{\alpha} \ttf{\delta}}}$) $\bar{\delta}$ $\preceq \bar{\alpha}$
\end{enumerate}
then we have the following properties:
\begin{enumerate}
    \item Bounded feedback signals: $|y[t]| \preceq \bar{y}$, $|u[t]| \preceq \bar{u}$, $|\alpha[t]| \preceq \bar{\alpha}$, $|\delta[t]| \preceq \bar{\delta}$ for all $t \geq 0$
    \item Bounded state: $|x[t]| \preceq$ $\bar{x}$, with $\bar{x}$ = \text{abs}($\tf{\Phi_{xw}}$) $\bar{w}$ + \text{abs}($\tf{\Phi_{xu}}$) $\bar{u}$ + \text{abs}($\tf{\Phi_{x\ttf{\delta}}}$) $\bar{\delta}$ for all $t \geq 0$
\end{enumerate}
\end{theorem}
The key idea of Theorem \ref{thm:LTI} is to combine a static neural network certification algorithm~\cite{kolter2017provable,weng2018towards,Gehr2018AI2,dvijotham2018dual,Boopathy2019cnncert, royo2019fast} with robust control theory to certify a neural network policy in a feedback control loop. The first condition of Theorem \ref{thm:LTI} is an input-output characterization of the neural network policy, which can be obtained by a static neural network certification algorithm. The second condition is a characterization of the model uncertainty block $\ttf{\Delta}$. The third condition, when combining with the first two conditions, ensures that $\{(y, u, \alpha, \delta) | \,\, |y| \preceq \bar{y}, |u| \preceq \bar{u}, |\alpha| \preceq \bar{\alpha}, |\delta| \preceq \bar{\delta} \}$ is a \emph{positively invariant set} of the closed loop dynamical system. The complete proof of Theorem \ref{thm:LTI} is in Appendix A. Theorem \ref{thm:LTI} can be used as follows: if we have $(\bar{x}, \bar{y}, \bar{u}) \preceq (x_{lim}, y_{lim}, u_{lim})$, then the requirement \eqref{eq:goal} is satisfied. We will discuss how to find a quadruplet $(\bar{y}, \bar{u}, \bar{\alpha}, \bar{\delta})$ satisfying the conditions of Theorem \ref{thm:LTI} in Section \ref{sec:find_invariant}. 

Theorem \ref{thm:LTI} has several advantages over Lemma \ref{lemma:L1}. First, Theorem \ref{thm:LTI} does not require the differentiability or continuity of the neural network policy $\pi(\cdot)$. Theorem \ref{thm:LTI} is valid as long as the property $|u| \preceq \bar{u}$ for all $|y| \preceq \bar{y}$ can be certified. This is one of the key difference between our approach and the existing literature~\cite{berkenkamp2017safe, richards2018lyapunov, jin2018stability, jin2018control}, where the results are obtained based on the assumption of Lipschitz continuity. Second, the neural network certification tool can give a tighter input-output characterization of the neural network policy than the local Lipschitz constant. As a result, the conditions of Theorem \ref{thm:LTI} are less restricted and easier to satisfy. Indeed, the following Theorem (proof in Appendix A) claims that the conditions of Lemma \ref{lemma:L1} implies that the sufficient conditions of Theorem \ref{thm:LTI} will always hold.  This means that Theorem \ref{thm:LTI} can be applied on a strictly larger class of problems than that of Lemma \ref{lemma:L1}.
\begin{theorem} \label{thm:tight}
The conditions of Lemma \ref{lemma:L1} imply the conditions of Theorem \ref{thm:LTI}.
\end{theorem}

Note that Theorem \ref{thm:LTI} only certifies the \emph{boundedness} of the closed loop system, not the \emph{stability} of the closed loop system -- Theorem \ref{thm:LTI} does not guarantee that $x = 0$ is a stable equilibrium. However, in the presence of \emph{persistent} adversarial perturbation, we argue that there is no significant difference between boundedness and stability because both of them will have a finite yet nonzero state deviation. In our case study in Section \ref{sec:experiment}, we show that Theorem \ref{thm:LTI} is much more useful than the traditional robust control approach (Lemma \ref{lemma:L1}) as it can certify the requirement \eqref{eq:goal} with persistent adversarial attack \eqref{eq:attack} that is $5$ times larger.

\subsection{Unstable Plant} \label{sec:open_unstable}
In this subsection, we consider the case where the plant \eqref{eq:NLTI1} - \eqref{eq:NLTI3} is unstable. Our strategy is to extract a first order approximation of the neural network policy to stabilize the plant first, then analyze the interconnection of the stabilized plant and the residual control policy. Specifically, we rewrite the neural network policy as $u[t] = \pi(y[t]) = K_0 y[t] + \pi_0(y[t])$ for some matrix $K_0$. We call $\pi_0(\cdot)$ the residual control policy. Equations \eqref{eq:NLTI1} - \eqref{eq:NLTI3} then become
\begin{subequations}
\begin{align}
x[t+1] &= (A + B K_0 C) x[t] + B u_0[t] + (B K_0 D_w + B_w) w[t] + B_\delta \delta[t] \label{eq:LLTI1}\\
y[t] &= C x[t] + D_w w[t] \label{eq:LLTI2} \\
\alpha[t] &= (C_{\alpha} + D_{\alpha u} K_0 C ) x[t] + D_{\alpha u} u_0[t] + (D_{\alpha u} K_0 D_w + D_{\alpha w}) w[t] \label{eq:LLTI3}
\end{align}
\end{subequations}
with the neural network policy $u_0[t] = \pi_0(y[t])$. As long as the spectral radius of the closed loop system matrix $A_{cl} = (A + B K_0 C)$ is less than $1$, the transfer matrix $(zI - A_{cl})^{-1}$ is stable. Theorem \ref{thm:LTI} can then be used with redefined transfer matrices: for instance, we have $\tf{\Phi_{xu}} = (zI-A_{cl})^{-1} B$, $\tf{\Phi_{xw}} = (zI-A_{cl})^{-1} (B K_0 D_w + B_w)$, and $\tf{\Phi_{x \ttf{\delta}}} = (zI-A_{cl})^{-1} B_\delta$. Other transfer matrices can be derived in a similar manner. In the following, we propose two different ways to obtain a candidate $K_0$: (1) using the Jacobian evaluated at the origin, and (2) using neural network certification tool to find a first order approximation of the policy over a region. 

If the neural network policy has the property $\pi(0) = 0$ and is differentiable at $y = 0$, we can use the Jacobian of $\pi(y)$ evaluated at $y = 0$ as a candidate for $K_0$, i.e., $K_0 = \frac{\partial \pi(y)}{\partial y}\Bigr|_{y=0}$. In this case, the residual control policy $\pi_0(y)$ has the following property:
\begin{equation}
    \pi_0(0) = 0 \quad \text{and} \quad \underset{\|y\| \to 0}{\text{limit}} \,\, \frac{\|\pi_0(y)\|}{\|y\|} \to 0 \label{eq:res_cont}
\end{equation}
We then have the following Lemma adopt from Theorem $4.3$ in \cite{astrom2010feedback} (uncertainty $\ttf{\Delta}$ is ignored).
\begin{lemma}[Local stability] \label{lemma:local}
Consider the dynamical system \eqref{eq:LLTI1} - \eqref{eq:LLTI2} with the residual control policy $\pi_0(y) = \pi(y) - K_0 y$ satisfying \eqref{eq:res_cont}. If $\rho(A + B K_0 C) < 1$, then $x = 0$ is a locally asymptotically stable equilibrium point of the system \eqref{eq:LLTI1} - \eqref{eq:LLTI3}. On the other hand, if $\rho(A + B K_0 C) > 1$, then $x = 0$ is a locally unstable equilibrium point of the system \eqref{eq:LLTI1} - \eqref{eq:LLTI3}.
\end{lemma}

If the neural network policy is not differentiable or the equilibrium $x = 0$ is locally unstable according to Lemma \ref{lemma:local}, one can leverage the neural network certification tools to find a different candidate $K_0$. For instance, the method proposed in~\cite{weng2018towards} provide a pair of linear lower and upper bounds on the neural network output as
\begin{equation}
   K_L y + b_L \leq \pi(y) \leq K_U y + b_U, \quad \forall | y | \preceq y_{ref}. \label{eq:NN-CERT}
\end{equation}
A candidate $K_0$ is given by $(K_U + K_L)/2$, which is a first order approximation of the control policy over the region $| y | \preceq y_{ref}$. 
As long as we can find a $K_0$ to make $\rho(A_{cl}) < 1$, Theorem \ref{thm:LTI} is a valid sufficient condition to verify the requirement \eqref{eq:goal} under persistent adversarial perturbation \eqref{eq:attack}. This statement holds even when a stable equilibrium does not exist.

\section{Algorithms}\label{sec:find_invariant}

In this section, we propose an iterative algorithm (Algorithm \ref{alg:lin}) to find a quadruplet $(\bar{y}, \bar{u_0}, \bar{\alpha}, \bar{\delta})$ that satisfies the conditions of Theorem \ref{thm:LTI}. We then propose a simple and effective attack strategy for a neural network control system.
\subsection{Algorithm for Theorem \ref{thm:LTI}}
For an unstable plant, we assume that a matrix $K_d$ with $\rho(A + B K_d C) < 1$ is given as a default linear approximation of the neural network policy\footnote{We note that finding a static gain $K$ to make $(A + B K C)$ a stable matrix is known as the static output feedback problem in the control literature. In the worst case, this problem can be NP-hard \cite{blondel1997np}. For an output feedback problem, one may consider using a dynamic (recurrent) policy for control, but that is beyond the scope of this paper.}. 
We assume that the uncertainty block is described by $|\delta| \preceq \Gamma_{\Delta} |\alpha|$ for a non-negative matrix $\Gamma_{\Delta}$. We explain why this form of uncertainty is natural when the model and uncertainty are learned using a data driven approach in~\cite{dean2017sample} in Appendix B. The following algorithm finds a quadruplet $(\bar{y}, \bar{u_0}, \bar{\alpha}, \bar{\delta})$ satisfying the conditions of Theorem \ref{thm:LTI}.

\begin{algorithm}[ht]
   \caption{Certification of state and control constraints under persistent adversarial perturbation}
   \label{alg:lin}
\begin{algorithmic}[1]
   \STATE {\bfseries Input:}  Equations \eqref{eq:NLTI1} - \eqref{eq:goal}, initial bounds $y_{ref} = 0$ and $\alpha_{ref} = 0$, default approximation $K_d$, parameter $\epsilon = 10^{-6}$, flag Success = False, Done = False, $k = 0$, MaxIter = $200$
   \STATE {\bfseries Output:} flag Success, certified bounds $\bar{x}$, $\bar{y}$, $\bar{u}$
   \WHILE{Done == False \AND $k$ < MaxIter}
   \STATE $\bar{u_0}$, $\bar{u}$, $K$ = NN-CERTIFICATION($\pi$, $y_{ref}$, $K_d$) from existing tools such as \eqref{eq:NN-CERT}
   \STATE $\bar{\delta} = \Gamma_{\Delta} \,\, \alpha_{ref}$
   \STATE Calculate the transfer matrices $\tf{\Phi_{xu}}$, $\tf{\Phi_{xw}}$, $\tf{\Phi_{x \ttf{\delta}}}$, $\tf{\Phi_{yu}}$, $\tf{\Phi_{yw}}$, $\tf{\Phi_{y \ttf{\delta}}}$, $\tf{\Phi_{\ttf{\alpha} u}}$, $\tf{\Phi_{\ttf{\alpha} w}}$, $\tf{\Phi_{\ttf{\alpha} \ttf{\delta}}}$.
   \STATE $\begin{bmatrix} \bar{x} \\ \bar{y} \\ \bar{\alpha} \end{bmatrix} =$ abs\Big($\begin{bmatrix} \tf{\Phi_{xw}} & \tf{\Phi_{xu}} & \tf{\Phi_{x \ttf{\delta}}} \\ \tf{\Phi_{yw}} & \tf{\Phi_{yu}} & \tf{\Phi_{y \ttf{\delta}}} \\ \tf{\Phi_{\ttf{\alpha} w}} & \tf{\Phi_{\ttf{\alpha} u}} & \tf{\Phi_{\ttf{\alpha}  \ttf{\delta}}} \end{bmatrix}$\Big) $\begin{bmatrix} \bar{w} \\ \bar{u_0} \\ \bar{\delta} \end{bmatrix}$
   \IF {$\bar{x} \not \preceq x_{lim}$ or $\bar{y} \not \preceq y_{lim}$ or $\bar{u} \not \preceq u_{lim}$}
   \STATE Success = False, Done = True
   \ELSIF {$\bar{y} \preceq y_{ref}$ and $\bar{\alpha} \preceq \alpha_{ref}$}
   \STATE Success = True, Done = True
   \ELSE
   \STATE $y_{ref} = (1+\epsilon) \bar{y}$, $\alpha_{ref} = (1+\epsilon) \bar{\alpha}$, $k = k+1$
   \ENDIF
   \ENDWHILE
   \STATE \textbf{Return} Success, $\bar{x}$, $\bar{y}$, $\bar{u}$
\end{algorithmic}
\end{algorithm}

Algorithm \ref{alg:lin} can be interpreted as follows. 
We first extract a linear policy $K$ to make the closed loop transfer matrices stable and calculate the bounds for $\bar{u}$ and $\bar{u_0}$ over the region $|y| \preceq y_{ref}$ using any neural network certification tool. Meanwhile, we calculate the bound for $\delta$ based on the assumption of the uncertainty block $\ttf{\Delta}$. These two steps ensure that the quadruplet $(y_{ref}, \bar{u_0}, \alpha_{ref}, \bar{\delta})$ satisfies the first two conditions of Theorem \ref{thm:LTI}. Given the range of the adversarial perturbation $\bar{w}$, the residual control action $\bar{u_0}$, and the uncertainty-induced input $\bar{\delta}$, we calculate the bounds for the state $x$, measurement $y$, and $\alpha$ on Line $7$. If we have $\bar{y} \preceq y_{ref}$ and $\bar{\alpha} \preceq \alpha_{ref}$, then $(y_{ref}, \bar{u_0}, \alpha_{ref}, \bar{\delta})$ also satisfies the third condition of Theorem \ref{thm:LTI}, thus we obtain a certificate for closed loop boundedness. If $\bar{y} \not\preceq y_{ref}$ or $\bar{\alpha} \not\preceq \alpha_{ref}$, we then increase the test bound $(y_{ref}, \alpha_{ref})$ and repeat the search. 


\subsection{Attack Algorithm}\label{sec:attack_algo}
In addition to the certification algorithm, here we propose an algorithm to attack the neural network policy in the control loop. As an example, we show how to design a perturbation sequence $\{w[t]\}_{t=0}^T$ with $\| \tf w \|_{\ell_\infty} = 1$ to attack the $i$-th state $\tf x_i$. Our idea is to follow Lines $4$ - $6$ of Algorithm \ref{alg:lin} to construct the closed loop transfer matrices with the help of neural network certification tools. We then have
\begin{equation}
    x[T] = \sum_{\tau=0}^T \Phi_{xw}[T-\tau] w[\tau] + \Phi_{xu}[T-\tau] u_0[\tau] + \Phi_{x \delta}[T-\tau] \delta[\tau]. \nonumber
\end{equation}
If we ignore the contribution of $u_0$ and $\delta$, we can maximize $x_i[T]$ using the following $\ell_\infty$ bounded perturbation sequence:
\begin{equation}
    w_j[t] = \text{sign}((\Phi_{xw}[T-t])_{ij}), \quad t = 0, \cdots, T \label{eq:wattack}
\end{equation}
for each $j$. 
Note that \eqref{eq:wattack} is sub-optimal because we ignore the contribution of $u_0$ and $\delta$. Nevertheless, we show in the next section that our attack \eqref{eq:wattack} is extremely strong.

\section{Case Study} \label{sec:experiment}

In this section, we demonstrate our algorithm on a cart-pole example and show the following results:
\begin{itemize}
    \item[(i)] Our proposed framework (Theorem \ref{thm:LTI} and Algorithm \ref{alg:lin}) outperforms methods based on traditional robust control theory (Lemma \ref{lemma:L1}). Specifically, Algorithm \ref{alg:lin} can certify the boundedness of the closed loop system with attack level that is $5$ times larger than that of a robust control approach.
    See the result in Figure \ref{fig:a}.
    \item[(ii)] Our proposed attack algorithm \eqref{eq:wattack} is far more effective than an exhuastive Monte-Carlo attack. In particular, our model-based attack algorithm can successfully discover the internal vulnerability of the closed loop system while an exhaustive Monte-Carlo simulation cannot. See the result in Figures \ref{fig:b}, \ref{fig:atkin}, and \ref{fig:atkout}.
    \item[(iii)] We show that Algorithm \ref{alg:lin} works well on situations where the system dynamics is \emph{unknown}, \emph{unstable}, and \emph{nonlinear}. We use the technique in~\cite{dean2017sample} to learn a nominal plant model $\tf P$     with conservative over-approximation of the modeling error $\ttf{\Delta}$, then use Algorithm \ref{alg:lin} to certify the robustness of the neural network policy interconnected with the learned model. The result is close to that with the knowledge of the true model. See the result in Figure \ref{fig:learned_model}.
    \item[(iv)] We show that Algorithm \ref{alg:lin} can be applied on a discontinuous and non-Lipschitz neural network policy, where the control action is quantized into discrete levels. Lipschitz-based methods~\cite{berkenkamp2017safe, richards2018lyapunov, jin2018stability, jin2018control} cannot be used to certify the boundedness of a closed loop system in this case. See the result in Figure \ref{fig:non_lips_policy}.
\end{itemize}


\paragraph{Experiment setup.} We use proximal policy optimization~\cite{schulman2017proximal} in stable baselines~\cite{stable-baselines} to train a $3$-layer neural network policy for the cart-pole problem in Open-AI gym~\cite{openai-gym}. 
Our neural network has $16$ neurons per hidden layer, ReLU activations, and continuous control output. We note that many trained neural network policies can obtain perfect reward in the training environment over a finite horizon, but will eventually become unstable if we simulate the closed loop dynamics over a longer time horizon. 
To obtain a more stable policy, we modify the reward function to be
\begin{equation}
    r[t] = 2 - \frac{x[t]^\top Q x[t] + u[t]^\top R u[t]}{x_{lim}^\top Q x_{lim} + u_{lim}^\top R u_{lim}}, \nonumber
\end{equation}
which will attempt to minimize the quadratic cost $x[t]^\top Q x[t] + u[t]^\top R u[t]$. 
The policy is trained with $2$M steps. We use the neural network certification framework recently developed in~\cite{weng2018towards, Boopathy2019cnncert, zhang2018crown} but any other neural network certification tools such as ~\cite{katz2017reluplex, Gehr2018AI2, dvijotham2018dual} can be used in Algorithm \ref{alg:lin} as well. The stronger the neural network certification algorithm, the stronger the robustness certificate delivered by Algorithm \ref{alg:lin}. 

For the ease of illustrating the result, we consider an one dimensional persistent perturbation on the pole angle measurement of the cart-pole. The requirement is to certify that a single state (angle of the pole) is within the user-specified limit. Note that Algorithm \ref{alg:lin} can be applied to systems with multi-dimensional perturbations with user-specified requirement on both state, measurement, and control action. In the following experiments, we use \emph{attack level} $w_\infty$ to represent the $\ell_\infty$ norm of the perturbation on pole angle measurement, and use \emph{state deviation limit} $x_{lim}$ to represent the limit on the $\ell_\infty$ norm of the pole angle state. Both the attack level and the state deviation are normalized with respect to $0.014$ degree, which is the maximum attack level that can be certified by the Lipschitz-based robust control approach (Lemma \ref{lemma:L1}). The complete model equations are in Appendix B.


\begin{figure}[t!]
\caption{\textbf{Experiment I}: Compare our methods with traditional approaches. \textbf{(a)}: The safe region certified by our Algorithm \ref{alg:lin} (area below the blue curve) is larger than the safe region certified by the traditional robust control approach (Lemma \ref{lemma:L1}, area below the red curve). \textbf{(b)}: The unsafe region discovered by our attack algorithm \eqref{eq:wattack} (area above the black curve) is much larger than the unsafe region discovered by exhaustive Monte-Carlo simulation (area above the brown curve).}
    \centering
    \begin{subfigure}[t]{0.48\textwidth}
         \includegraphics[width=\textwidth]{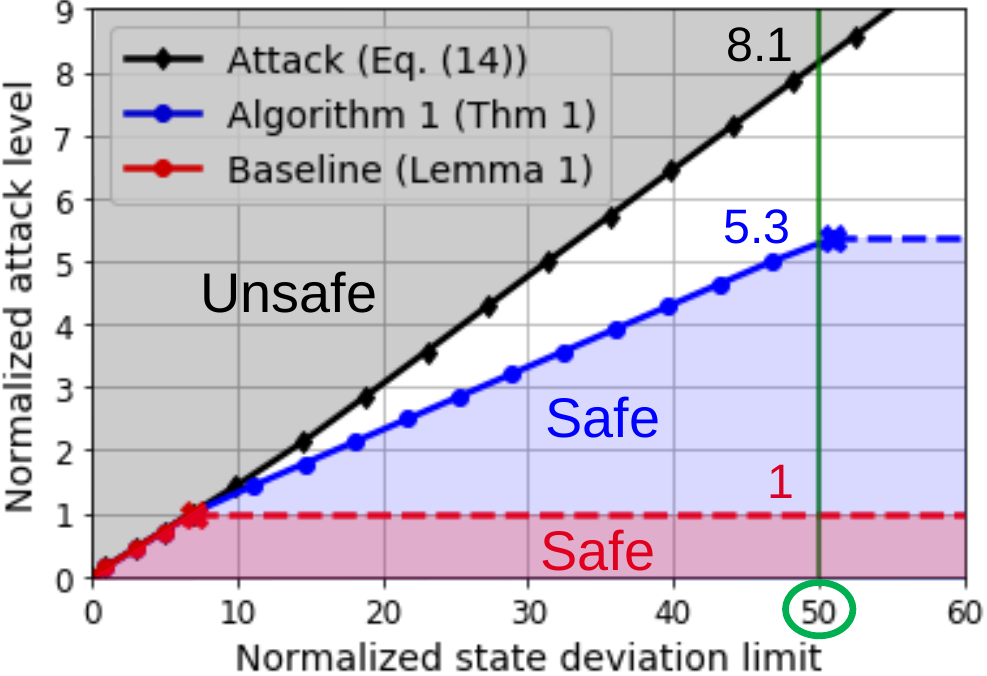}
         \vspace{1pt}
         \caption{Algorithm \ref{alg:lin} can be 5.3$\times$ better than the robust control approach.}
         \label{fig:a}
     \end{subfigure}%
     ~
     \begin{subfigure}[t]{0.48\textwidth}
         \includegraphics[width=\textwidth]{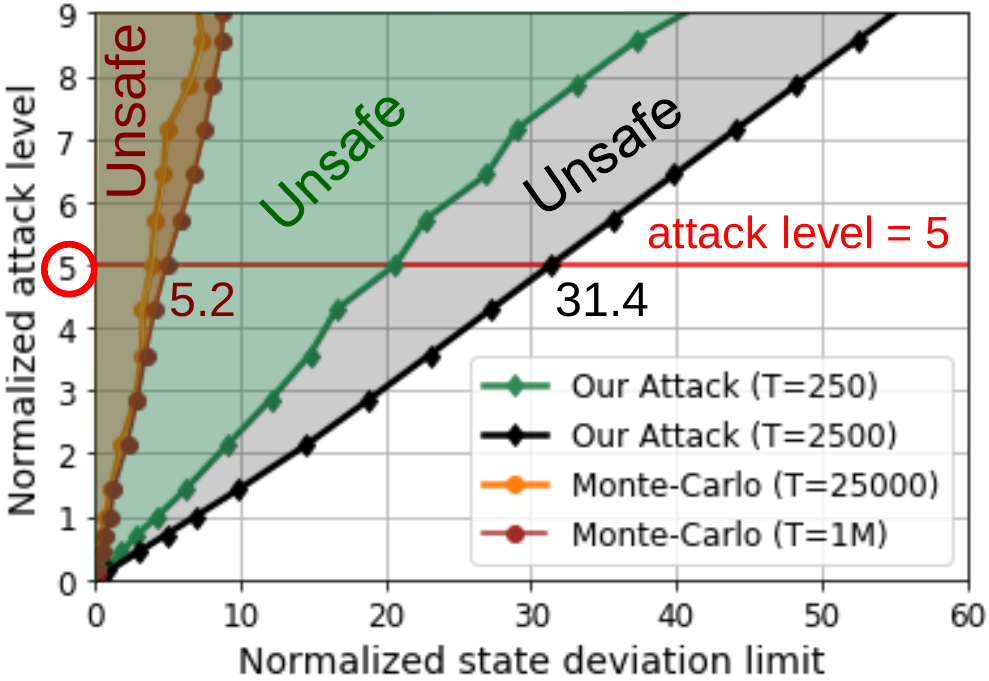}
         \vspace{1pt}
         \caption{Our attack \eqref{eq:wattack} is far more effective than the exhaustive Monte-Carlo simulation.}
         \label{fig:b}
     \end{subfigure}%
\end{figure}


\subsection{Experiment I: Compare with traditional robust control certification and Monte-Carlo based attacks} \label{sec:exp1} 
\paragraph{Tightness of Algorithm \ref{alg:lin}.}
In this experiment, we use the linearized cart-pole model with no uncertainty ($\Gamma_\Delta = 0$) to compare the tightness of Algorithm \ref{alg:lin} (Theorem \ref{thm:LTI}) and the Lipschitz-based robust control baseline (Lemma \ref{lemma:L1}). Given a user-specified state deviation limit $x_{lim}$, we use binary search to call Algorithm \ref{alg:lin} repeatedly to obtain the largest possible $w_\infty^*$ such that the safety requirement $\|\tf x\|_{\ell_\infty} \leq x_{lim}$ is satisfied for any persistent attack $\|\tf w\|_{\ell_\infty} \leq w_\infty^*$. We plot $w_\infty^*$ as a function of $x_{lim}$ as the blue curve in Figure \ref{fig:a}. Clearly, the area below the blue curve is the safe region certified by Theorem \ref{thm:LTI}. Likewise, the area below the red curve is the safe region certified by the traditional robust control theory (Lemma \ref{lemma:L1}), where the local Lipschitz constant is obtained via a sampling-based approach\footnote{We note that several recent works~\cite{fazlyab2019efficient} have proposed algorithms to calculate the Lipschitz constant of a neural network policy. However, since cart-pole is an unstable system, we need to use the technique described in Section \ref{sec:open_unstable} and the Lipschitz constant we need is the \emph{local} Lipschitz constant of the \emph{residual} neural network policy $u_0 = \pi_0(y) = \pi(y) - K_0 y$. We are not aware of any work that can give us a tight bound in this case, thus we use sampling-based approach to obtain a lower bound of the local Lipschitz constant. Strictly speaking, since it is only a lower bound, Lemma \ref{lemma:L1} does not offer a safety certificate. Even though the comparison is in favor of the robust control baseline, we still show that our Algorithm \ref{alg:lin} outperforms the robust control baseline with a huge margin.}. Figure \ref{fig:a} validates our claim in Theorem \ref{thm:tight} -- the safe region certified by Lemma \ref{lemma:L1} is \emph{always} a subset of the safe region certified by Theorem \ref{thm:LTI}. For $x_{lim} = 50$ (the vertical green line in Figure \ref{fig:a}), Algorithm \ref{alg:lin} can certify an attack level that is $5.3$ times larger than the one using Lemma \ref{lemma:L1}. Note that there is a flat dashed line at attack level $= 1$ for the robust control approach. This is because when the attack level is greater than $1$, the conditions \eqref{eq:L1-1} - \eqref{eq:L1-3} of Lemma \ref{lemma:L1} no longer hold. In other words, the maximum attack level that can be certified by Lemma \ref{lemma:L1} is $1$ regardless of the state deviation limit.

Note that in this experiment, the local Lipschitz constant used in Lemma \ref{lemma:L1} is only a lower bound because it is obtained via a sampling-based approach. Therefore, the only reason that can explain the gap between the blue curve and the red curve is that the \emph{neural network certification algorithm~\cite{weng2018towards} gives a much tighter input-output characterization of the neural network policy than the Lipschitz-based method} (which uses $\| \tf u \|_{\ell_\infty} \leq \gamma_\pi \|\tf y \|_{\ell_\infty}$). Other Lipschitz-based methods also have this limitation~\cite{berkenkamp2017safe, richards2018lyapunov, jin2018stability, jin2018control}. In brief, our Algorithm \ref{alg:lin} outperforms the robust control baseline and can achieve up to $5.3$ times better robustness certificate. 


In Figure \ref{fig:a}, the area above the black curve is the unsafe region where our attack algorithm \eqref{eq:wattack} (with $T = 2500$) is able to make the state deviation exceed the user-specified limit. When the attack level is small ($< 1$), we can see from Figure \ref{fig:a} that the black, blue, and red curves overlap. This means that our certification bound is tight in this region. When the attack level increases (or the state deviation limit increases), we start to see a gap between our certification algorithm and our attack algorithm. Future research will attempt to further bridge this gap.

\begin{figure}[t!]
\caption{Comparison between our designed attack with Monte-Carlo based attack. \textbf{(a)}: Our designed attack is injected between $100$s and $150$s. Monte-Carlo based random attack is injected before $100$s and after $150$s. \textbf{(b)}: Our designed attack causes a huge state deviation between $100$s and $150$s compared to the Monte-Carlo based attack.}
    \centering
    \begin{subfigure}[t]{0.45\textwidth}
         \includegraphics[width=\textwidth]{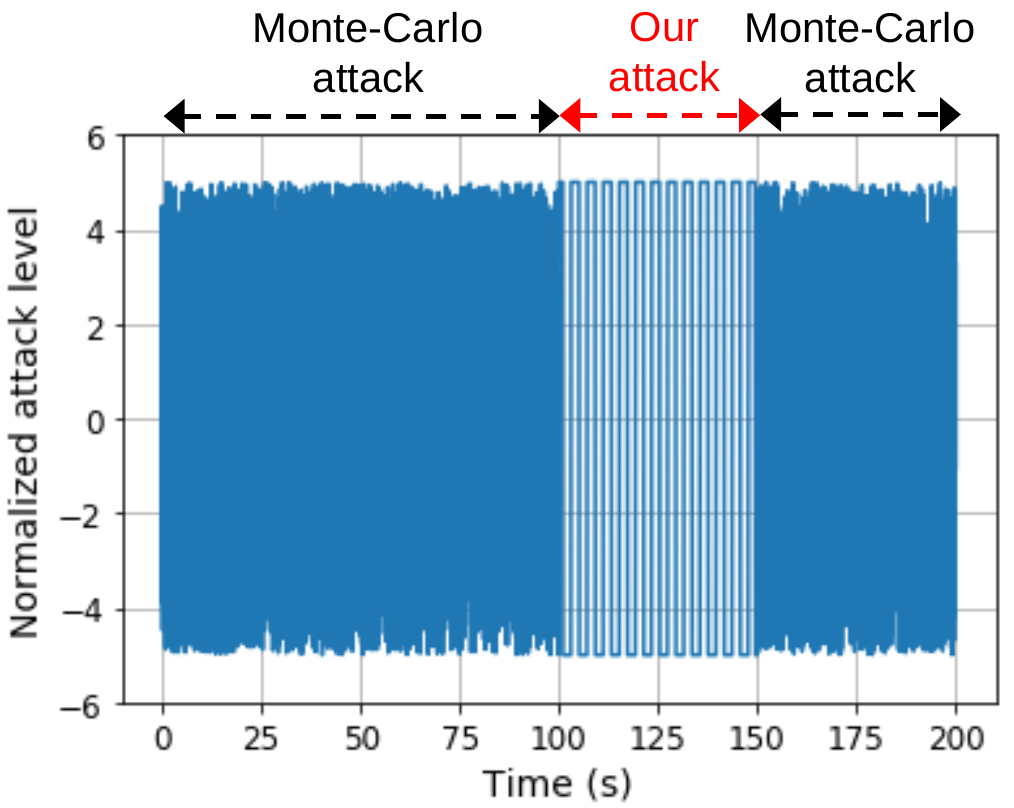}
         \caption{Perturbation sequence as the input}
         \label{fig:atkin}
     \end{subfigure}%
     ~
     \begin{subfigure}[t]{0.45\textwidth}
         \includegraphics[width=\textwidth]{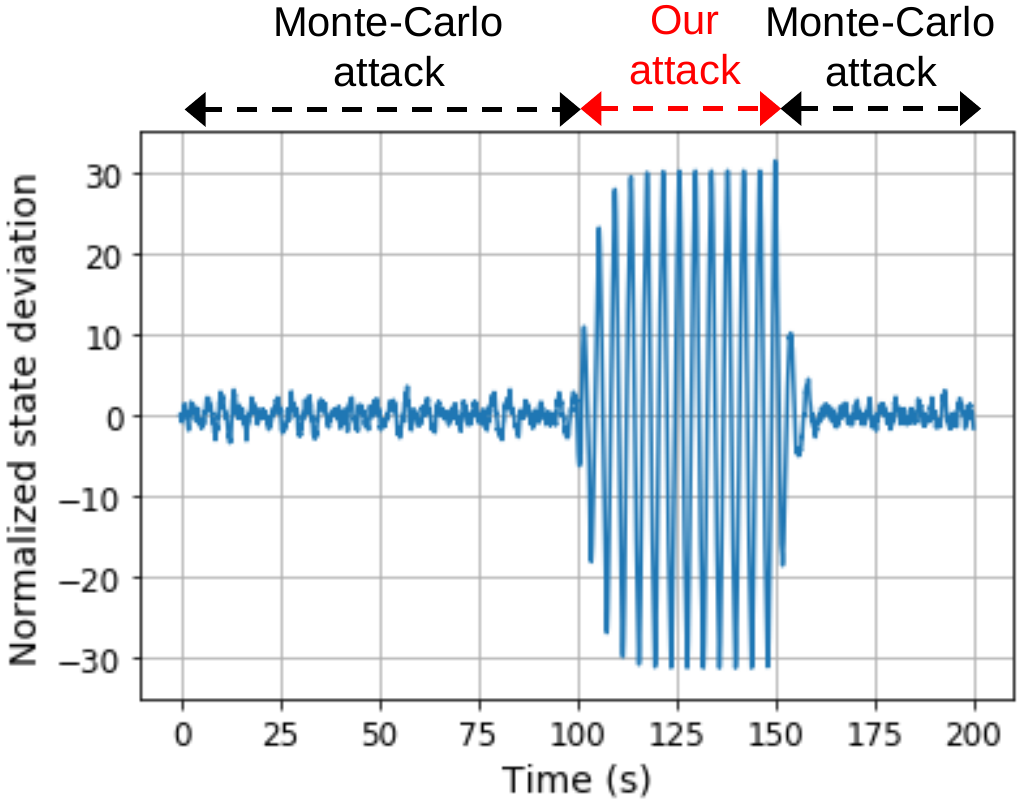}
         \caption{State deviation as the output}
         \label{fig:atkout}
     \end{subfigure}%
\end{figure}

\textbf{Our attack algorithm.} In Figure \ref{fig:b}, we show that our simple attack algorithm \eqref{eq:wattack} is far more effective than an exhaustive Monte-Carlo simulation. Specifically, the area above the black curve is the unsafe region discovered by our attack algorithm \eqref{eq:wattack} (with $T = 2500$), while the area above the brown curve is the unsafe region discovered by Monte-Carlo simulation with $1$ million time steps. There is a huge gap between the unsafe region found by the two methods. This is a strong evidence that the safety of a neural network control system cannot be certified using exhaustive Monte-Carlo simulation. We need to use the certification framework developed in this paper to guarantee the safeness of a neural network control system.

As a concrete example, at attack level $5$ (the horizontal red line in Figure \ref{fig:b}), the mean, standard deviation, and maximum of the state deviation in $1$ million steps Monte-Carlo simulation are $0$, $1.1$, and $5.2$, respectively.
One may conclude that the probability of seeing a state deviation greater than $5.2$ is $10^{-6}$ and falsely claim that the area below the brown region is safe. Unfortunately, this statement is not true when the perturbation sequence is adversarial -- for the same level of perturbation, the maximum state deviation found by our algorithm is $31.4$, which is $29$ standard deviation away. We show the perturbation sequence (input) and the state deviation (output) as functions of time in Figures \ref{fig:atkin} and \ref{fig:atkout}. In Figure \ref{fig:atkin}, we inject our designed attack between $100$s and $150$s (sampling time $= 0.02$s, and thus the number of time step is $T = 2500$). We inject Monte-Carlo based random attack with the same attack level before $100$s and after $150$s. It is clear from Figure \ref{fig:atkout} that the state deviation is significantly higher when we inject our attack. Specifically, our attack algorithm can excite the resonance of the closed loop system, while an exhaustive Monte-Carlo simulation usually cannot. We should also note that here we only consider an \emph{one dimensional} perturbation sequence. Our algorithm \eqref{eq:wattack} can be used to design a multi-dimensional perturbation. The gap between our approach and the Monte-Carlo approach will be even larger in the multi-dimensional setting. 

In summary, Figures \ref{fig:a} and \ref{fig:b} show that both our certification algorithm and attack algorithm are significantly better than the traditional methods.



\begin{figure}[t!]
\caption{\textbf{Experiment II}: Algorithm \ref{alg:lin} works on learned models and non-Lipschitz neural network policies. \textbf{(a)}: Compare Algorithm 1 on various learned models~\cite{dean2017sample} and the true model. (b) Compare certificates on Lipschitz and non-Lipschitz (quantized) neural network policies.}
    \centering
    \begin{subfigure}[t]{0.48\textwidth}
         \includegraphics[width=\textwidth]{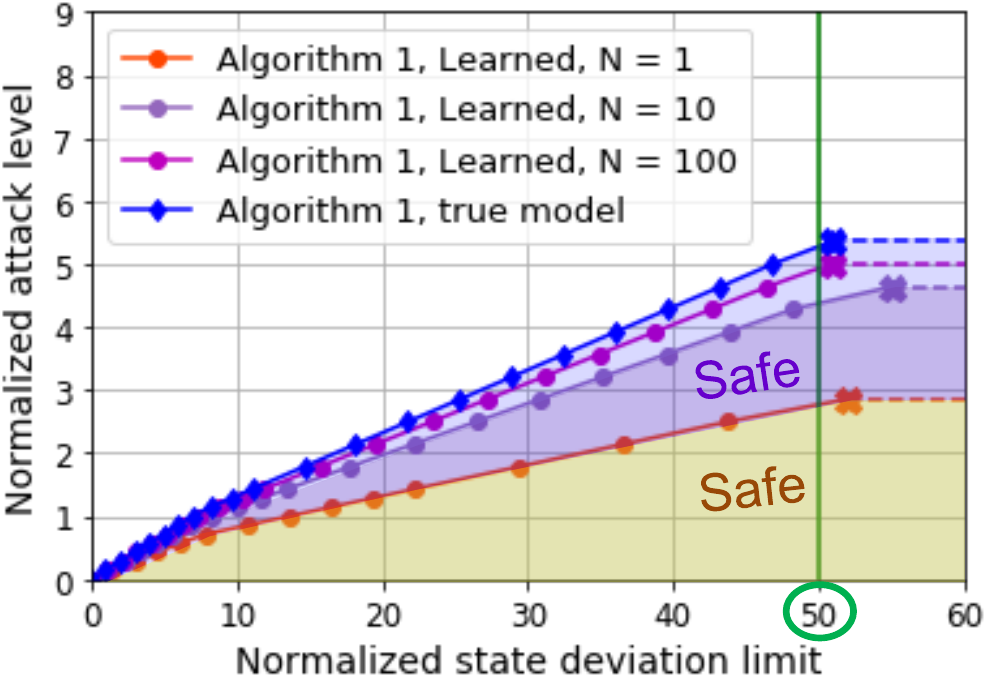}
         \caption{Algorithm \ref{alg:lin} works on learned models.}
         \label{fig:learned_model}
     \end{subfigure}%
     ~
     \begin{subfigure}[t]{0.48\textwidth}
         \includegraphics[width=\textwidth]{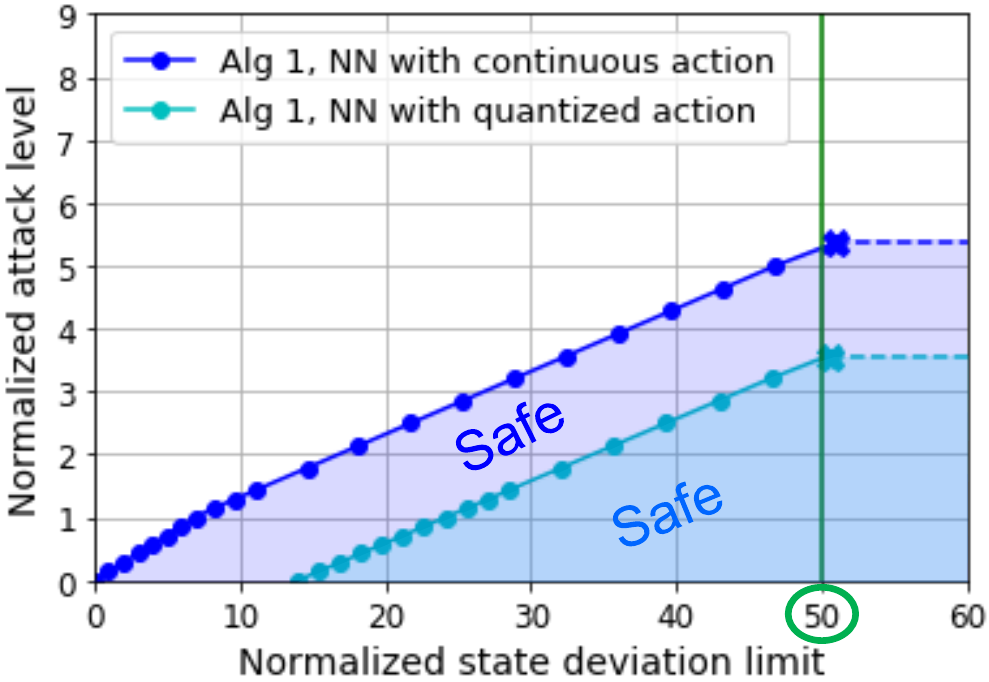}
         \caption{Algorithm \ref{alg:lin} works on non-Lipschitz policies.}
         \label{fig:non_lips_policy}
     \end{subfigure}%
\end{figure}

\subsection{Experiment II: Certificates on learned dynamics and non-Lipschitz policies} \label{sec:exp2} 
\paragraph{Algorithm \ref{alg:lin} on learned models.} In many reinforcement learning applications, the mathematical model of the system dynamics \eqref{eq:NLTI1} - \eqref{eq:NLTI5} is not available in the first place. Therefore, we need to learn the model before performing the certification task. In this experiment, we show that Algorithm \ref{alg:lin} works well on learned models -- even when the underlying system dynamics is \emph{unknown}, \emph{unstable}, and \emph{nonlinear}.

We use the method proposed in~\cite{dean2017sample} to learn the nominal model $\tf P$ and a conservative estimate of the uncertainty $\ttf{\Delta}$. For each episode, we run the simulation using the black-box nonlinear cart-pole model with random control action for $30$ time steps. We run $N$ episodes of simulation to collect the data, then we solve a regression problem to estimate the system matrices for the nominal model $\tf P$. 
To get a conservative bound of the modeling error $\mathbf{\Delta}$, we synthesize $100$ bootstrap samples of $\tf P$ and use the element-wise maximum deviation from the nominal model to over-approximate the model uncertainty $\ttf{\Delta}$. Empirically, there is high probability that the modeling error is bounded by $\ttf{\Delta}$. Beside the bootstrap approach, one can also use other algorithms proposed in~\cite{dean2017sample} to get a stronger mathematical guarantee on the bounds of modeling error if certain technical conditions hold. We include a detailed description of our model learning procedure in Appendix B.

In Figure \ref{fig:learned_model}, we show the certification result of Algorithm \ref{alg:lin} on models learned with different number of episodes $N$. As $N$ increases, the $100$ bootstrap samples become more consistent and thus the size of the modeling error $\mathbf{\Delta}$ shrinks. For $N = 100$, we can see from Figure \ref{fig:learned_model} that the safe region certified by Algorithm \ref{alg:lin} is very close to that with the true linearized model with no uncertainty. This experiment shows that our method indeed works well even when the model of the system dynamics is unknown in the first place. Given a nonlinear and unknown plant model interconnected with a neural network control policy, as long as we can find a nominal LTI plant $\tf P$ and bound the modeling error by $\ttf{\Delta}$, we can use Algorithm \ref{alg:lin} to certify the boundedness of the closed loop system under persistent adversarial perturbation.


\paragraph{Algorithm \ref{alg:lin} on non-Lipschitz neural network policies.} In Figure \ref{fig:non_lips_policy}, we show that Algorithm \ref{alg:lin} can be applied on a non-Lipschitz and discontinuous neural network policy. Specifically, we assume that the output of the neural network is quantized into discrete levels. This is a common setting in many reinforcement learning tasks, in which the control action is chosen from a discrete set. When the neural network output is quantized into discrete levels, the neural network policy becomes discontinuous and the closed loop system may not have a stable equilibrium. Therefore, the Lipschitz-based and the stability based methods~\cite{berkenkamp2017safe, richards2018lyapunov, jin2018stability, jin2018control} cannot be used to certify the boundedness of the closed loop system. On the other hand, our Algorithm \ref{alg:lin} can still be applied in this case -- the area below the light blue curve in Figure \ref{fig:non_lips_policy} is the safe region certified by Algorithm \ref{alg:lin}. This is because we use the static neural network certification tools~\cite{kolter2017provable,weng2018towards,Gehr2018AI2,dvijotham2018dual,Boopathy2019cnncert, royo2019fast} to characterize the input-output relation of a neural network policy, which works even when the policy is not Lipschitz continuous.

\section{Conclusions and Future Works}
Neural networks have shown superior performance in various control tasks in reinforcement learning~\cite{zhang2016learning}, yet people have concerns using them in safety critical systems because it is hard to certify its robustness under adversarial attacks. 
In this paper, we extended the neural network certification tools~\cite{katz2017reluplex, kolter2017provable,weng2018towards,Gehr2018AI2,dvijotham2018dual,Boopathy2019cnncert, royo2019fast} into a dynamic setting and developed an algorithm to certify the robustness of a neural network policy in a feedback control loop under persistent adversarial attack. 
We showed both theoretically and empirically that our method outperforms the traditional Lipschitz-based robust control approach and works on situations where the model dynamics is unknown in the first place. We also developed an $\ell_\infty$ attack algorithm and showed that it can discover the vulnerability of a neural network control system while an exhaustive Monte-Carlo simulation cannot -- this suggested that a mathematical-based certification framework, like the one developed in this paper, is necessary to ensure the safety of a neural network control system. The key idea of our method is combining static neural network certification algorithms with robust control using invariant set principle. This idea is fundamental and can be extended to more general settings in many directions. In the future, we will further tighten the certification bound by exploring more general characterization of the invariant set.




\clearpage

\bibliography{ref}
\bibliographystyle{ieeetr} 

\newpage

\section*{Appendix A: proof of Theorems}
Here we give the proofs for Lemma \ref{lemma:L1} and Theorems \ref{thm:LTI} and \ref{thm:tight}. We also introduce Corollary \ref{lemma:Hinf}, which is the global $\mathcal{H}_\infty$ version of Lemma \ref{lemma:L1}. 
\subsection*{Proof of Lemma \ref{lemma:L1}}
\begin{proof}[Proof of Lemma \ref{lemma:L1}]
We first show that \eqref{eq:L1-1} - \eqref{eq:L1-2} is a sufficient condition for global stability when $\gamma_\pi$ is the global Lipschitz constant. From the assumption, we have
\begin{subequations}
\begin{align}
\| \ttf{\delta} \|_{\ell_\infty} & \leq \gamma_{\Delta} \| \ttf{\alpha} \|_{\ell_\infty} \nonumber\\
&= \gamma_{\Delta} \| \tf{\Phi_{\ttf{\alpha} u}} \tf u + \tf{\Phi_{\ttf{\alpha} w}} \tf w + \tf{\Phi_{\ttf{\alpha} \ttf{\delta}}} \ttf{\delta} \|_{\ell_\infty} \nonumber\\
&\leq \gamma_{\Delta} \| \tf{\Phi_{\ttf{\alpha} \ttf{\delta}}} \|_{\mathcal{L}_1} \| \ttf{\delta} \|_{\ell_\infty} + \gamma_{\Delta} \| \tf{\Phi_{\ttf{\alpha} u}} \tf u + \tf{\Phi_{\ttf{\alpha} w}} \tf w \|_{\ell_\infty}. \nonumber
\end{align}
\end{subequations}
From condition \eqref{eq:L1-1}, we can bound the $\ell_\infty$ norm of $\ttf{\delta}$ as
\begin{equation}
\| \ttf{\delta} \|_{\ell_\infty} \leq \frac{\gamma_{\Delta}}{1- \beta_1} \| \tf{\Phi_{\ttf{\alpha} u}} \tf u + \tf{\Phi_{\ttf{\alpha} w}} \tf w \|_{\ell_\infty}.    
\end{equation}
Then, we have
\begin{subequations}
\begin{align}
\| \tf y \|_{\ell_\infty} & = \| \tf{\Phi_{y \ttf{\delta}}}\ttf{\delta} + \tf{\Phi_{yu}} \tf u + \tf{\Phi_{yw}} \tf w\|_{\ell_\infty} \label{eq:der1}\\
& \leq \| \tf{\Phi_{y \ttf{\delta}}}\|_{\mathcal{L}_1} \|\ttf{\delta}\|_{\ell_\infty} + \| \tf{\Phi_{yu}} \tf u + \tf{\Phi_{yw}} \tf w\|_{\ell_\infty} \label{eq:der2}\\
&\leq \frac{\gamma_{\Delta}}{1 - \beta_1} \| \tf{\Phi_{y \ttf{\delta}}}\|_{\mathcal{L}_1} \| \tf{\Phi_{\ttf{\alpha} u}} \tf u + \tf{\Phi_{\ttf{\alpha} w}} \tf w \|_{\ell_\infty} + \| \tf{\Phi_{yu}} \tf u + \tf{\Phi_{yw}} \tf w\|_{\ell_\infty} \label{eq:der3}\\
&\leq \Big(\|\tf{\Phi_{yu}}\|_{\mathcal{L}_1} + \frac{\gamma_{\Delta}}{1 - \beta_1} \| \tf{\Phi_{y \ttf{\delta}}}\|_{\mathcal{L}_1} \|\tf{\Phi_{\ttf{\alpha} u}}\|_{\mathcal{L}_1} \Big) \| \tf u \|_{\ell_\infty} \nonumber\\
& \quad + \|\tf{\Phi_{yw}} \tf w \|_{\ell_\infty} + \frac{\gamma_{\Delta}}{1- \beta_1} \| \tf{\Phi_{y \ttf{\delta}}}\|_{\mathcal{L}_1} \|\tf{\Phi_{\ttf{\alpha} w}} \tf w \|_{\ell_\infty} \label{eq:der4}\\
&\leq \gamma_\pi \Big(\|\tf{\Phi_{yu}}\|_{\mathcal{L}_1} + \frac{\gamma_{\Delta}}{1- \beta_1} \| \tf{\Phi_{y \ttf{\delta}}}\|_{\mathcal{L}_1} \|\tf{\Phi_{\ttf{\alpha} u}}\|_{\mathcal{L}_1} \Big) \| \tf y \|_{\ell_\infty} \nonumber\\
& \quad + \Big(\|\tf{\Phi_{yw}} \|_{\mathcal{L}_1} + \frac{\gamma_{\Delta}}{1- \beta_1} \| \tf{\Phi_{y \ttf{\delta}}}\|_{\mathcal{L}_1} \|\tf{\Phi_{\ttf{\alpha} w}} \|_{\mathcal{L}_1} \Big) \| \tf w \|_{\ell_\infty} \label{eq:der5}
\end{align}
\end{subequations}
From condition \eqref{eq:L1-2}, we can bound the $\ell_\infty$ norm of $\tf y$ as
\begin{equation}
\| \tf y \|_{\ell_\infty} \leq \frac{1}{1-\beta_2} \Big(\|\tf{\Phi_{yw}} \|_{\mathcal{L}_1} + \frac{\gamma_{\Delta}}{1- \beta_1} \| \tf{\Phi_{y \ttf{\delta}}}\|_{\mathcal{L}_1} \|\tf{\Phi_{\ttf{\alpha} w}} \|_{\mathcal{L}_1} \Big) \| \tf w \|_{\ell_\infty}. \label{eq:yw}
\end{equation}
This shows that the $\ell_\infty$ to $\ell_\infty$ gain from $\tf w$ to $\tf y$ is bounded. We can use similar procedure to show that the $\ell_\infty$ to $\ell_\infty$ gain from $\tf w$ to $\tf u$, $\ttf{\delta}$, $\ttf{\alpha}$ are all bounded. This shows the input-output stability of the closed loop system when the global Lipschitz constant of the neural network policy is $\gamma_\pi$.

Next, we consider the case where $\gamma_\pi$ is valid only over a local region $\|y\|_{\ell_\infty} \leq y_\infty$. 
From \eqref{eq:L1-3}, we know that the right-hand-side of \eqref{eq:yw} is less than or equal to $y_\infty$ given $\| \tf w \|_{\ell_\infty} \leq w_\infty$. Thus, $\tf y$ will never go outside the local region $\|y\|_{\ell_\infty} \leq y_\infty$ where we calculate the Lipschitz constant $\gamma_\pi$ for any valid perturbation \eqref{eq:attack}. This completes the proof. 
\end{proof}

\subsection*{The global $\mathcal{H}_\infty$ version of Lemma \ref{lemma:L1}}
The following Corollary is the global $\mathcal{H}_\infty$ version of Lemma \ref{lemma:L1}. 
\begin{corollary} \label{lemma:Hinf}
Consider a stable LTI plant \eqref{eq:NLTI1} - \eqref{eq:zero} interconnected with a neural network policy \eqref{eq:NLTI4} and a dynamic uncertainty block \eqref{eq:NLTI5} as shown in Figure \ref{fig:model}. Assume that the persistent perturbation $\tf w$ lies in the set given by \eqref{eq:attack}. Suppose that the neural network policy $u = \pi(y)$ has a finite $\ell_2$ to $\ell_2$ gain $\gamma_\pi$ for all $y$, and the uncertainty block $\ttf{\Delta}$ has the property $\| \ttf{\delta} \|_{\ell_2} \leq \gamma_\Delta \|\ttf{\alpha} \|_{\ell_2}$. If the following conditions hold:
\begin{subequations}
\begin{align}
& \beta_1 = \gamma_{\Delta} \| \tf{\Phi_{\ttf{\alpha} \ttf{\delta}}} \|_{\mathcal{H}_\infty} < 1 \label{eq:Hinf-1}\\
& \beta_2 = \gamma_\pi \Big[\|\tf{\Phi_{yu}}\|_{\mathcal{H}_\infty} + \frac{\gamma_{\Delta}}{1 - \beta_1} \| \tf{\Phi_{y \ttf{\delta}}}\|_{\mathcal{H}_\infty} \|\tf{\Phi_{\ttf{\alpha} u}}\|_{\mathcal{H}_\infty} \Big] < 1 \label{eq:Hinf-2}
\end{align}
\end{subequations}
then the closed loop system in Figure \ref{fig:model} is $\ell_2$ to $\ell_2$ input-output stable.
\end{corollary}


\subsection*{Proof of Theorem \ref{thm:LTI}}
The proof of Theorem \ref{thm:LTI} relies on the concept of a positively invariant set, which is defined as follows:
\begin{definition}[from \cite{khalil2002nonlinear}]
A set $\SigSp{M}$ is said to be a \emph{positively invariant set} with respect to the dynamics $x[t+1] = f(x[t])$ if
\begin{equation}
    x[0] \in \SigSp{M} \implies x[t] \in \SigSp{M}, \forall t \geq 0.
\end{equation}
\end{definition}
The proof of Theorem \ref{thm:LTI} is given as follows:
\begin{proof}[Proof of Theorem \ref{thm:LTI}]
We use mathematical induction to show that $\SigSp{I} = \{(y, u, \alpha, \delta) | \,\, |y| \preceq \bar{y}, |u| \preceq \bar{u}, |\alpha| \preceq \bar{\alpha}, |\delta| \preceq \bar{\delta} \}$ is a positively invariant set of the closed loop dynamical system if the three conditions of the theorem are given. From the zero initial condition assumption in \eqref{eq:zero}, we have
\begin{equation}
|y[0]| = |D_w w[0]| \preceq \text{abs}(D_w) \bar{w} \preceq \text{abs}(\tf{\Phi_{yw}}) \bar{w} \preceq \bar{y} \nonumber 
\end{equation}
where the last inequality is from the third condition of the theorem. Then from the first condition of the theorem, we have $|u[0]| \preceq \bar{u}$. Similarly, we have
\begin{equation}
|\alpha[0]| = |D_{\alpha u} u[0] + D_{\alpha w} w[0]| \preceq \text{abs}(D_{\alpha u}) \bar{u} + \text{abs}(D_{\alpha w}) \bar{w} \preceq \text{abs}(\tf{\Phi_{\ttf{\alpha}u}}) \bar{u} + \text{abs}(\tf{\Phi_{\ttf{\alpha}w}}) \bar{w} \preceq \bar{\alpha}. \nonumber
\end{equation}
Then from the second condition of the theorem, we have $|\delta[0]| \preceq \bar{\delta}$. This shows $(y[0], u[0], \alpha[0], \delta[0]) \in \SigSp{I}$. Assume that we have $(y[t], u[t], \alpha[t], \delta[t]) \in \SigSp{I}$ for all $ 0 \leq t < T$. From equation \eqref{eq:FNLTI2}, we have
\begin{subequations}
\begin{align}
|y[T]| &= |\sum_{\tau=0}^\infty \Phi_{yu}[\tau] u[T-\tau] + \Phi_{yw}[\tau] w[T-\tau] + \Phi_{y \delta}[\tau] \delta[T-\tau] \, | \nonumber\\
&\preceq \sum_{\tau=0}^\infty | \Phi_{yu}[\tau] u[T-\tau] | + |\Phi_{yw}[\tau] w[T-\tau]| + |\Phi_{y \delta}[\tau] \delta[T-\tau] \, | \nonumber \\
&\preceq \sum_{\tau=0}^\infty |\Phi_{yu}[\tau]| \bar{u} + |\Phi_{yw}[\tau]| \bar{w} + |\Phi_{y \delta}[\tau]| \bar{\delta} \nonumber \\
&= \text{abs}(\tf{\Phi_{yu}}) \bar{u} + \text{abs}(\tf{\Phi_{yw}}) \bar{w} + \text{abs}(\tf{\Phi_{y \ttf{\delta}}}) \bar{\delta} \nonumber \\
&\preceq \bar{y} \nonumber
\end{align}
\end{subequations}
where the last inequality is from the third condition of the theorem. Similarly, we can derive $|\alpha[T]| \preceq \bar{\alpha}$ from \eqref{eq:FNLTI3}. 
The first two conditions of Theorem \ref{thm:LTI} then imply $|u[T]| \preceq \bar{u}$ and $|\delta[T]| \preceq \bar{\delta}$, and thus we have $(y[T], u[T], \alpha[T], \delta[T]) \in \SigSp{I}$. Using mathematical induction, we conclude that $(y[t], u[t], \alpha[t], \delta[t]) \in \SigSp{I}$ for all $t \geq 0$ and the closed loop feedback signals and state are bounded within the set specified by the theorem.
\end{proof}


\subsection*{Proof of Theorem \ref{thm:tight}}
\begin{proof}[Proof of Theorem \ref{thm:tight}]
We need to show that when \eqref{eq:L1-1} - \eqref{eq:L1-3} and the locally Lipschitz continuous assumption of Lemma \ref{lemma:L1} are satisfied, we can always construct a quadruplet $(\bar{y}, \bar{u}, \bar{\alpha}, \bar{\delta})$ satisfying the three conditions of Theorem \ref{thm:LTI}. 
Consider the following equation:
\begin{equation}
    \begin{bmatrix} y_{ref} \\ \alpha_{ref} \end{bmatrix} = \frac{1}{(1-\beta_1)(1-\beta_2)} \begin{bmatrix} 1 - \gamma_\Delta \| \tf{\Phi_{\ttf{\alpha} \ttf{\delta}}} \|_{\mathcal{L}_1} & \gamma_\Delta \|\tf{\Phi_{y \ttf{\delta}}}\|_{\mathcal{L}_1} \\ \gamma_\pi \|\tf{\Phi_{\ttf{\alpha} u}}\|_{\mathcal{L}_1} & 1 - \gamma_\pi \|\tf{\Phi_{yu}}\|_{\mathcal{L}_1} \end{bmatrix} \begin{bmatrix} \|\tf{\Phi_{yw}} \|_{\mathcal{L}_1} \\ \|\tf{\Phi_{\ttf{\alpha} w}}\|_{\mathcal{L}_1} \end{bmatrix} w_\infty, \label{eq:ext_p2}
\end{equation}
with scalar variables $y_{ref}$ and $\alpha_{ref}$. For any $w_\infty > 0$, we have $y_{ref} > 0$ and $\alpha_{ref} > 0$ because all the elements in \eqref{eq:ext_p2} are positive according to the conditions \eqref{eq:L1-1} - \eqref{eq:L1-2}. In addition, we can show that $y_{ref}$ defined in \eqref{eq:ext_p2} is equivalent to the left-hand-side of \eqref{eq:L1-3}. Therefore, we have $y_{ref} \leq y_\infty$ from \eqref{eq:L1-3} -- this means that $y_{ref}$ is always contained within the region where we calculate the local Lipschitz constant of the neural network policy $\gamma_\pi$. It is then straightforward to verify that $\bar{y} = y_{ref} \mathbf{1}$, $\bar{u} = \gamma_\pi y_{ref} \mathbf{1}$, $\bar{\alpha} = \alpha_{ref} \mathbf{1}$, $\bar{\delta} = \gamma_\Delta \alpha_{ref} \mathbf{1}$ satisfy the first two conditions of Theorem \ref{thm:LTI} given the locally Lipschitz continuous assumption of Lemma \ref{lemma:L1}.

For the third condition of Theorem \ref{thm:LTI}, it can be verified that \eqref{eq:ext_p2} is a solution to the following inequality:
\begin{equation}
    \begin{bmatrix} \|\tf{\Phi_{yw}}\|_{\mathcal{L}_1} \\ \|\tf{\Phi_{\ttf{\alpha} w}} \|_{\mathcal{L}_1} \end{bmatrix} w_\infty \preceq \begin{bmatrix}1 - \gamma_\pi \|\tf{\Phi_{yu}}\|_{\mathcal{L}_1} & -\gamma_\Delta \|\tf{\Phi_{y \ttf{\delta}}}\|_{\mathcal{L}_1} \\ -\gamma_\pi \|\tf{\Phi_{\ttf{\alpha} u}}\|_{\mathcal{L}_1} & 1 - \gamma_\Delta \| \tf{\Phi_{\ttf{\alpha} \ttf{\delta}}} \|_{\mathcal{L}_1}  \end{bmatrix} \begin{bmatrix} y_{ref} \\ \alpha_{ref} \end{bmatrix}, \nonumber
\end{equation}
which can be rearranged into
\begin{equation}
    \begin{bmatrix} \|\tf{\Phi_{yw}}\|_{\mathcal{L}_1} \\ \|\tf{\Phi_{\ttf{\alpha} w}} \|_{\mathcal{L}_1} \end{bmatrix} w_\infty + \begin{bmatrix} \|\tf{\Phi_{yu}}\|_{\mathcal{L}_1} & \|\tf{\Phi_{y \ttf{\delta}}}\|_{\mathcal{L}_1} \\ \|\tf{\Phi_{\ttf{\alpha} u}}\|_{\mathcal{L}_1} & \| \tf{\Phi_{\ttf{\alpha} \ttf{\delta}}} \|_{\mathcal{L}_1}  \end{bmatrix} \begin{bmatrix} \gamma_\pi y_{ref} \\ \gamma_\Delta \alpha_{ref} \end{bmatrix} \preceq \begin{bmatrix} y_{ref} \\ \alpha_{ref} \end{bmatrix}. \label{eq:ext_p1}
\end{equation}
Finally, we note the following inequality
\begin{equation}
    \text{abs}(\tf{\Phi}) \mathbf{1} \preceq \|\tf{\Phi}\|_{\mathcal{L}_1} \mathbf{1}
\end{equation}
from the fact that the $\mathcal{L}_1$ norm is selecting the maximum row sum. Therefore, we have
\begin{subequations}
\begin{align}
& \text{abs}\Big( \begin{bmatrix} \tf{\Phi_{yw}} \\ \tf{\Phi_{\ttf{\alpha} w}} \end{bmatrix}\Big) \bar{w} + \text{abs}\Big( \begin{bmatrix} \tf{\Phi_{yu}} & \tf{\Phi_{y \ttf{\delta}}} \\ \tf{\Phi_{\ttf{\alpha} u}} & \tf{\Phi_{\ttf{\alpha} \ttf{\delta}}}  \end{bmatrix}\Big) \begin{bmatrix} \bar{u} \\ \bar{\delta} \end{bmatrix} \nonumber\\
\preceq \quad & \begin{bmatrix} \|\tf{\Phi_{yw}}\|_{\mathcal{L}_1}\mathbf{1} \\ \|\tf{\Phi_{\ttf{\alpha} w}}\|_{\mathcal{L}_1}\mathbf{1} \end{bmatrix} w_\infty + \text{abs}\Big( \begin{bmatrix} \tf{\Phi_{yu}} & \tf{\Phi_{y \ttf{\delta}}} \\ \tf{\Phi_{\ttf{\alpha} u}} & \tf{\Phi_{\ttf{\alpha} \ttf{\delta}}}  \end{bmatrix}\Big) \begin{bmatrix} \gamma_\pi y_{ref} \mathbf{1} \\ \gamma_\Delta \alpha_{ref} \mathbf{1} \end{bmatrix} \nonumber\\
\preceq \quad & \begin{bmatrix} \|\tf{\Phi_{yw}}\|_{\mathcal{L}_1}\mathbf{1} \\ \|\tf{\Phi_{\ttf{\alpha} w}}\|_{\mathcal{L}_1}\mathbf{1} \end{bmatrix} w_\infty + \begin{bmatrix} \|\tf{\Phi_{yu}}\|_{\mathcal{L}_1}\mathbf{1} & \|\tf{\Phi_{y \ttf{\delta}}}\|_{\mathcal{L}_1}\mathbf{1} \\ \|\tf{\Phi_{\ttf{\alpha} u}}\|_{\mathcal{L}_1}\mathbf{1} & \| \tf{\Phi_{\ttf{\alpha} \ttf{\delta}}} \|_{\mathcal{L}_1}\mathbf{1} \end{bmatrix} \begin{bmatrix} \gamma_\pi y_{ref} \\ \gamma_\Delta \alpha_{ref} \end{bmatrix} \nonumber\\
\preceq \quad & \begin{bmatrix} y_{ref} \mathbf{1} \\ \alpha_{ref} \mathbf{1} \end{bmatrix} \nonumber\\
= \quad &  \begin{bmatrix} \bar{y} \\ \bar{\alpha} \end{bmatrix}. \nonumber
\end{align}
\end{subequations}

We can see that if the conditions of Lemma \ref{lemma:L1} hold, we can always construct a quadruplet $(\bar{y}, \bar{u}, \bar{\alpha}, \bar{\delta})$ satisfying all the three conditions of Theorem \ref{thm:LTI}. The converse is not true. Therefore, Theorem \ref{thm:LTI} can be applied on a strictly larger class of problems than Lemma \ref{lemma:L1}.
\end{proof}

\newpage 

\section*{Appendix B: cart pole model}

We consider a cart-pole problem with $\eta$ being the displacement of the cart and $\theta$ the angle of the pole. The dynamics is given by
\begin{subequations}
\begin{align}
\Ddot{\eta} &= \Big(\frac{4}{3}(M+m)l - ml \text{cos}^2(\theta)\Big)^{-1} (\frac{4}{3}ml^2 \dot{\theta}^2 \text{sin}(\theta) - mgl\text{sin}(\theta)\text{cos}(\theta) + \frac{4}{3}lu )\nonumber\\
\Ddot{\theta} &= \Big(\frac{4}{3}(M+m)l - ml \text{cos}^2(\theta)\Big)^{-1}(-ml\dot{\theta}^2\text{sin}(\theta)\text{cos}(\theta) + (M+m)g\text{sin}(\theta) - \text{cos}(\theta)u). \nonumber
\end{align}
\end{subequations}
We use the default model parameters from stable baselines: $g = 9.8$, $M = 1$, $m = 0.1$, and $l = 0.5$. Using Euler discretization, the nonlinear discrete time cart pole model is given by
\begin{equation}
    \eta[t+1] = \eta[t] + \tau \dot{\eta}[t], \, \dot{\eta}[t+1] = \dot{\eta}[t] + \tau \Ddot{\eta}[t], \, \theta[t+1] = \theta[t] + \tau \dot{\theta}[t], \, \dot{\theta}[t+1] = \dot{\theta}[t] + \tau \Ddot{\theta}[t] \nonumber 
\end{equation}
with sampling time $\tau = 0.02$. 

\subsection*{Linearized model for Experiment I (Section \ref{sec:exp1})}
Define the state vector $x = \begin{bmatrix} \eta & \dot{\eta} & \theta & \dot{\theta}\end{bmatrix}^\top$. The linearized cart-pole model around the origin is given by
\begin{equation}
    x[t+1] = A x[t] + B u[t], \quad y[t] = x[t] + D_w w[t] \nonumber
\end{equation}
with
\begin{equation}
A = \begin{bmatrix}1 & \tau & 0 & 0 \\ 0 & 1 & \frac{-3mg\tau}{4M+m} & 0 \\ 0 & 0 & 1 & \tau \\ 0 & 0 & \frac{3(M+m)g\tau}{(4M+m)l} & 1\end{bmatrix}, \quad B = \begin{bmatrix} 0 \\ \frac{4\tau}{4M+m} \\ 0 \\ \frac{-3\tau}{(4M+m)l}\end{bmatrix}, \quad D_w = \begin{bmatrix} 0 \\ 0 \\ 1 \\ 0\end{bmatrix}. \label{eq:linear_disc}
\end{equation}
Note from $D_w$ that we have an one dimensional perturbation on the pole angle measurement. The requirement is to certify that the actual pole angle is within the user-specified limit, that is,
\begin{equation}
    | x_3[t] | \leq x_{lim}, \quad \forall t \geq 0.
\end{equation}
As explained in Section \ref{sec:exp1}, these assumptions are made for the ease of illustrating the result. Our method can work on much more general model with the form defined in Section \ref{sec:formulation}.

\subsection*{Learned model for Experiment II (Section \ref{sec:exp2})}
When the model equation is unknown in the first place, we can collect the data from the simulator and fit the data to a model. Let $x^{(i, j)}[t]$ be the state vector $x$ at time $t$ for the $i$-th episode from the $j$-th bootstrap sample, for $j = 0, 1, \dots, 100$. For the $j$-th bootstrap run, we solve the following least square problem to obtain the system matrices $A^{(j)}$ and $B^{(j)}$:
\begin{equation}
    \underset{A^{(j)},B^{(j)}}{\text{minimize}} \sum_{i = 1}^N \sum_{t = 0}^{T-1} \| x^{(i, j)}[t+1] - A^{(j)} x^{(i, j)}[t] - B^{(j)} u^{(i, j)}[t] \|_2^2 \label{eq:least_square}
\end{equation}
with $T = 30$, $x$ from the simulator, and $u$ randomly generated. We then find a pair of non-negative matrices $\Delta_A$ and $\Delta_B$ such that
\begin{equation}
    |A^{(j)} - A^{(0)}| \preceq \Delta_A, \quad |B^{(j)} - B^{(0)}| \preceq \Delta_B, \quad \text{for} \quad j = 1, \cdots, 100
\end{equation}
to over-approximate the modeling error of the nominal model $(A^{(0)}, B^{(0)})$. The learned model used by our experiment in Section \ref{sec:exp2} is then given by
\begin{eqnarray}
x[t+1] &=& A^{(0)} x[t] + B^{(0)} u[t] + \delta[t] \nonumber\\
y[t] &=& x[t] + D_w w[t] \nonumber \\
\alpha[t] &=& \begin{bmatrix} I \\ 0\end{bmatrix} x[t] + \begin{bmatrix} 0 \\ I \end{bmatrix} u[t] = \begin{bmatrix} x[t] \\ u[t] \end{bmatrix} \nonumber \\
|\delta[t]| &\preceq& \begin{bmatrix} \Delta_A & \Delta_B \end{bmatrix} |\alpha[t]| = \Gamma_\Delta |\alpha[t]|, \nonumber
\end{eqnarray}
where $\Gamma_\Delta = \begin{bmatrix} \Delta_A & \Delta_B \end{bmatrix}$ is the non-negative matrix used in Algorithm \ref{alg:lin}.

\end{document}